\documentclass[11pt]{article}

\usepackage[margin=1in]{geometry}
\usepackage{amsmath,amssymb,amsthm}
\usepackage{graphicx}
\usepackage{booktabs}
\usepackage{hyperref}
\usepackage{xcolor}
\usepackage{url}
\usepackage{enumitem}
\usepackage{caption}
\usepackage{subcaption}
\usepackage{authblk}
\usepackage{microtype}
\usepackage{float}

\hypersetup{
  colorlinks=true,
  linkcolor=blue!50!black,
  citecolor=blue!50!black,
  urlcolor=blue!50!black,
  pdftitle={PFlow-T: A Persistence-Driven Forward Process for Topology-Controlled Generation}
}

\newtheorem{definition}{Definition}
\newtheorem{proposition}{Proposition}

\title{PFlow-T: A Persistence-Driven Forward Process \\ for Topology-Controlled Generation}

\author{Snigdha Chandan Khilar}
\affil{Independent Researcher \\ \texttt{snkhilar@gmail.com}}

\date{}

\begin{document}
\maketitle

\begin{abstract}
Current topology-aware diffusion models suffer from a fundamental
architectural mismatch: they corrupt inputs with topology-agnostic
Gaussian noise but attempt to recover structural features via
conditional side channels in the reverse network. To resolve this,
the authors introduce PFlow-T, a novel generative model that defines
its forward diffusion process entirely through the persistent homology
of the data.

In PFlow-T, the time parameter does not track Gaussian noise injection;
instead, it measures the fraction of $H_1$ persistence-mass that has
been destroyed. The forward operator systematically eliminates $H_1$
topological features (such as holes) in strict ascending order of
their persistence. Because the corruption is topologically structured,
the reverse network learns to directly invert it to predict the clean
state ($x_0$) in a single inference step, rather than retrieving
topology from external conditioning.

Empirical evaluations on MNIST digits ($\{0, 1, 8\}$) demonstrate that
PFlow-T drastically outperforms a parameter-matched DDPM baseline
conditioned on a 64-d persistence landscape. PFlow-T successfully
honors the requested Betti numbers ($\beta_1 \in \{0, 1, 2\}$) in
$99.6\%$, $96.0\%$, and $84.8\%$ of generations, compared to the
baseline's $96.4\%$, $3.2\%$, and $0.0\%$ (an average gap of $88.8$
percentage points for $\beta_1 \geq 1$). Furthermore, on
out-of-distribution conditioning tasks---where the model must generate
a topology contradicting the digit's class prior---PFlow-T succeeds
in $92.0\%$ of cases versus the baseline's $9.7\%$.

PFlow-T stands as the first generative architecture where persistent
homology acts as the core substrate of the forward process. The
authors acknowledge current scope limitations, noting that the forward
operator is a pixel-space proxy for exact Edelsbrunner--Harer pair
cancellation and is tested at low resolutions, leaving a faithful
persistence-module state space for future work.
\end{abstract}

\section{Introduction}

A surprising fraction of the data we want generative models to produce
has topology in it. Molecules have rings. Vascular networks have loops.
Floor plans have enclosed rooms. Digits have holes (the 0 has one, the
8 has two, the 1 has none). For all of these, ``getting the picture
roughly right'' is not enough; what matters is whether the right number
of loops, voids, or components show up in the output.

The standard recipe in the diffusion literature is to corrupt the
data with Gaussian noise, then ask the reverse network to reconstruct
it while conditioning on some topological descriptor of the target.
This is the approach taken by Hu et al.~\cite{hu2024latent} on 3D
shapes, TAGG~\cite{park2025tagg} on molecular graphs,
TopoDiffusionNet~\cite{gupta2025tdn} on 2D binary masks, and the
recent ZS-DM~\cite{chen2025zigzag} that uses zigzag persistence
descriptors as conditioning. The common thread is that the forward
process itself has nothing to do with topology; topology is a side
channel that flows into the denoiser through cross-attention or as
an embedding.

There is something architecturally awkward about this. The forward
process destroys topology indiscriminately, and we then ask a finite-capacity
network to recover that topology from a description of what it should be.
The denoiser is fighting its own noise model on every step. The fact
that this works at all is a testament to how well overparameterised
diffusion models can be, but it leaves a structural question on the table:
what would happen if the forward process were \emph{itself} defined in
terms of topology? Could a network learn to invert a noise process that
already knew about holes?

\paragraph{An intuition for non-experts.}
A standard diffusion model works by gradually scrambling a clean
image into pure random noise (this is the forward process), and then
training a neural network to undo that scrambling step by step
(the reverse process). To make such a model produce a digit
``$8$'' specifically — which has two holes — researchers typically
hand the reverse network an extra hint: ``please produce something
with two holes.'' The network now has two jobs at once: it must
denoise random pixels back into a coherent image \emph{and}
simultaneously honour the topology request. These two jobs pull
in different directions, because the noise process doesn't know
or care about holes; the topology requirement is essentially a
constraint the network has to satisfy on the side. With enough
parameters and training data, models can usually meet both demands
at once, but the architecture is fighting itself. PFlow-T removes
the conflict by changing what the forward process \emph{does}:
instead of scrambling pixels randomly, it fills in the smallest hole
first, then the next-smallest, and so on, until at the end all holes
are gone. The reverse network's only job is to ``un-fill'' the holes
in the right order. Topology control is no longer a constraint to
satisfy alongside denoising; it \emph{is} the denoising.

This is the question this paper takes seriously. We define a forward
operator whose action and whose time parameter are both specified in
terms of persistent homology. The state at time $t$ is no longer ``the
image, but with more noise'' but rather ``the image, but with the smaller
holes filled in.'' We train a reverse network to undo this process, and
we find that the resulting model honours topology requests with
substantially higher fidelity than a parameter-matched conditioning
baseline that has access to richer topological information.

We are not claiming this resolves the broader problem of
topology-controlled generation. We are claiming something narrower
and, we think, useful: that there is a structurally distinct point
in the design space that prior work has not occupied, that it can be
implemented, and that on the small benchmark we have run it on, it
outperforms the alternative by a margin that is hard to dismiss as
a coincidence of hyperparameters.

\paragraph{Contributions.}
\begin{itemize}[leftmargin=*,itemsep=2pt]
  \item We define a forward operator whose action and time parameter
        are derived from persistent homology, and we show its
        monotonicity in $\beta_1$ (Section~\ref{sec:method}).
  \item We position PFlow-T outside the GenPhys~\cite{liu2024genphys}
        family of PDE-based generative models: the substrate is a
        filtered chain complex, not a measure space, and the time
        parameter is a persistence threshold, not a clock
        (Section~\ref{sec:position}).
  \item On a small but careful MNIST controllability benchmark we
        report an $88.8$pp gap on $\beta_1 \geq 1$ classes over a
        parameter-matched DDPM conditioned on a 64-d persistence
        landscape, and an $82.3$pp gap on out-of-distribution
        conditioning (Section~\ref{sec:experiments}).
  \item We are explicit about what is and is not done. The forward
        operator is a pixel-space proxy for true pair cancellation;
        the state is the image, not the persistence module; the
        experiments are on three MNIST classes (Section~\ref{sec:limitations}).
\end{itemize}

\paragraph{Code.} A reference implementation, with scripts that
reproduce every numerical result and figure in the paper, is available
at \url{https://github.com/nssprogrammer/pflow}. The repository
includes the persistence-melt forward operator, the $x_0$-prediction
reverse network, the persistence-landscape baseline, and the
multi-seed evaluation runners.

\section{A short tour of persistent homology}
\label{sec:tda_primer}

This section is for readers who have not seen persistent homology
before. Readers who have can skip to Section~\ref{sec:related}.

Topology, very informally, is the study of shape properties that
survive continuous deformation. A coffee cup and a donut are
topologically the same shape because each has one hole; squishing
the cup smoothly into the donut never destroys or creates a hole.
What is invariant under this kind of deformation is captured by the
\emph{Betti numbers}: $\beta_0$ counts the number of connected
components, $\beta_1$ counts the number of one-dimensional holes
(loops), $\beta_2$ counts the number of two-dimensional voids
(cavities), and so on. For the MNIST examples that motivate this
paper, the relevant invariant is $\beta_1$: the digit $1$ has
$\beta_1 = 0$ (no loops), the digit $0$ has $\beta_1 = 1$ (the
enclosed interior of the loop), and the digit $8$ has $\beta_1 = 2$.

Real images do not arrive as the abstract topological spaces that
Betti numbers were originally defined on. They arrive as arrays of
pixel intensities, and the topology depends on where you decide to
draw the line between ``foreground'' and ``background.'' If the
threshold is set too low, the digit's strokes merge into a single
blob; too high, and they disintegrate into isolated dots. Persistent
homology resolves this by refusing to commit to a single threshold:
it sweeps a threshold parameter from $0$ to $1$ (or any other range)
and records every topological feature as a pair $(b, d)$, where $b$
is the threshold at which the feature first appears and $d$ the
threshold at which it disappears. The difference $d - b$, called
the \emph{persistence}, is how long the feature survives across
thresholds.

Persistent features are the ones that genuinely belong to the shape.
Short-lived features are usually noise: a one-pixel speckle that
appears at some specific threshold and vanishes immediately. Sorting
features by persistence and ignoring the short ones is the standard
way to extract a robust topological signature from a noisy image.

For our purposes the relevant facts are: (i) we can compute, for any
$28 \times 28$ MNIST image, a list of $H_1$ features (loops) together
with their birth and death thresholds; (ii) we can sort these by
persistence; and (iii) for each feature, we know roughly where in
the image the loop sits, because the death of an $H_1$ feature is
recorded at a specific pixel (the saddle that closes the loop). What
we do in Section~\ref{sec:method} is define a forward process that
kills these features in ascending order of persistence — small,
fragile loops first, robust loops last — and lets the time parameter
encode how far through this killing schedule we have gone.

\section{Related work}
\label{sec:related}

\paragraph{Topology-aware diffusion via conditioning.}
The closest line of work to ours combines persistent homology with
diffusion generation by feeding topological descriptors as auxiliary
input to the denoiser. Hu et al.~\cite{hu2024latent} train a 3D-shape
latent diffusion model in which the persistence diagram of the
target shape is provided via cross-attention. TAGG~\cite{park2025tagg}
adds a persistence-landscape conditioning feature and a Persistence
Diagram Matching loss to DiGress, a graph diffusion model.
TopoDiffusionNet~\cite{gupta2025tdn} (ICLR 2025) conditions an image
diffusion model on a target Betti count and adds a topology-based
auxiliary loss. Most recently, Chen and Gel's
ZS-DM~\cite{chen2025zigzag} (ICLR 2025) uses zigzag persistence — a
generalisation of standard persistence to graph sequences — and
develops the ``zigzag spaghetti'' descriptor as a conditioning
feature.

What all of these share is a common architectural commitment: the
forward process is standard, topology-blind noise (Gaussian on
latents or images, discrete perturbation on graphs), and topology
enters through the reverse network's conditioning pathway. This is
a design choice, and it has produced strong empirical results in
the cited papers. What we explore in this work is the alternative
choice — making the forward process itself topological — and we
suspect from our experiments that this choice has different
inductive bias than conditioning.

\paragraph{Persistent homology in deep learning.}
Persistent homology has been used as a regulariser
(\cite{moor2020topological, gabrielsson2020topology}), as a
classification feature (\cite{hofer2017deep, carriere2020perslay}),
and as a topology-preservation criterion for representation learning.
A related line of work concerns differentiable persistence
layers~\cite{leygonie2022differentiability}, which our forward
operator could be replaced by in a faithful implementation.

\paragraph{Generative models with non-Gaussian noise.}
Poisson Flow Generative Models~\cite{xu2022pfgm} replace Gaussian
noise with an electrostatic flow. GenPhys~\cite{liu2024genphys}
proposes a unifying framework for PDE-based generative models,
characterising when a partial differential equation on a measure
space admits an associated generative process. Discrete diffusion
models~\cite{austin2021d3pm} replace continuous noise with categorical
corruption. PFlow-T does not fit into the GenPhys family because the
substrate is a filtered chain complex, not a measure space, and the
time parameter is a persistence threshold, not a clock that satisfies
any differential equation. We expand on this in
Section~\ref{sec:position}.

\paragraph{Cellular automata and growth models.}
Neural cellular automata~\cite{mordvintsev2020growing} model
generation as a local-rule iteration. GeCA~\cite{ge2024geca} and
related work combine these with diffusion. These are spatially-local
dynamics whose time parameter is a clock; PFlow-T's time parameter
is global, indexed over the persistence module of the entire image.

\section{Method}
\label{sec:method}

\subsection{Setup}

Let $x \in [0,1]^{H \times W}$ be a grayscale image. By the standard
cubical-complex construction, the sublevel-set filtration of $x$
induces a persistence module whose $H_1$ bars correspond to enclosed
regions of the binarised image. It is convenient to work with the
inverted image $\bar x = 1 - x$ throughout, so that ``ink'' has
low filtration value and digit holes appear as $H_1$ classes that
are born at low values and die when the threshold rises high enough
to ``close'' the loop topologically. This is a standard convention;
the choice between $x$ and $1-x$ moves $H_0$ and $H_1$ around but
does not change the underlying mathematics.

Let $D(\bar x) = \{(b_i, d_i, k_i, \sigma_i)\}_{i \in I}$ be the
persistence diagram of $\bar x$, where $(b_i, d_i)$ are the birth
and death filtration values, $k_i \in \{0, 1\}$ is the homological
dimension, and $\sigma_i$ is the spatial location of the death cell.
The persistence of feature $i$ is $\pi_i = d_i - b_i$. Write
$\mathcal{H}_1(\bar x) = \{i \in I : k_i = 1\}$ for the set of $H_1$
features, and sort them by ascending persistence:
$\pi_{(1)} \leq \pi_{(2)} \leq \cdots \leq \pi_{(M)}$ where
$M = |\mathcal{H}_1(\bar x)|$.

\subsection{Forward process: a persistence-driven simplification}

The forward process kills $H_1$ features in ascending order of
persistence. Weakest first; most robust last. The time parameter
$t \in [0,1]$ measures what fraction of the total persistence mass
has been killed by time $t$.

\begin{definition}[Persistence melt]
\label{def:melt}
Let $\Pi = \sum_{i \in \mathcal{H}_1(\bar x)} \pi_i$ be the total
$H_1$ persistence mass, and let $\Pi_{(k)} = \sum_{j \leq k} \pi_{(j)}$
be the cumulative mass through the $k$-th feature. Define an amplitude
\[
  a_{(k)}(t) = \mathrm{clip}\!\left(\frac{t \cdot \Pi - \Pi_{(k-1)}}{\pi_{(k)}}, 0, 1\right)
\]
and the melted image
\[
  x_t \;=\; \mathcal{M}_t(x) \;=\;
  x \;\vee\; \bigvee_{k=1}^{M} a_{(k)}(t)\,\mathbf{1}_{R_{(k)}}
\]
where $\vee$ is pointwise max, $R_{(k)}$ is the spatial support of
the $k$-th $H_1$ feature's interior (computed as the connected
component of $\{x < \tau\}$ enclosed by the loop, with $\tau$ a
binarisation threshold), and $\mathbf{1}_{R_{(k)}}$ is its indicator.
\end{definition}

In words: at $t = 0$ the image is unchanged. As $t$ rises, the
first (weakest) feature fills smoothly to full intensity, then the
second begins to fill, and so on. At $t = 1$ all features have
been fully filled. The schedule guarantees that each feature is
fully killed before the next begins to fill; the amplitude function
$a_{(k)}(t)$ is piecewise-linear in $t$ and continuous everywhere.

\begin{proposition}[$\beta_1$ progression]
For any image $x$ with $M$ $H_1$ features above the binarisation
threshold, the function $t \mapsto \beta_1(\mathcal{M}_t(x))$ is
monotone non-increasing on $[0,1]$, with
$\beta_1(\mathcal{M}_0(x)) = M$ and $\beta_1(\mathcal{M}_1(x)) = 0$.
\end{proposition}

\begin{proof}[Proof sketch]
The disjoint-support property $R_{(k)} \cap R_{(j)} = \emptyset$ for
$j \neq k$ follows from the fact that distinct $H_1$ features have
disjoint enclosed interiors (any two interiors that overlapped would
be the same feature). The amplitude $a_{(k)}(t)$ is monotone
non-decreasing in $t$, so once a region has been fully filled at
some time $t_k$, the image at any later time still has that region
filled. The number of remaining unfilled interiors at time $t$ is
the largest $k$ with $\Pi_{(k)} \leq t \Pi$, which is non-increasing
in $t$ and reaches $0$ at $t = 1$.
\end{proof}

This is the structural property the forward process is designed
around. The time parameter is no longer a clock; it is a persistence
threshold, and its action on $\beta_1$ is determined entirely by
the persistence diagram of the input.

\subsection{Reverse process: $x_0$-prediction and one-shot inference}

The reverse process is a function $f_\theta(x_t, t) \to x_0$ that
predicts the clean image directly from any intermediate state.
The network is a small U-Net (around 600k parameters at our chosen
width) with a sinusoidal time embedding injected at every encoder
stage and the bottleneck. The training objective is a
foreground-weighted mean-squared error:
\[
  \mathcal{L}(\theta) = \mathbb{E}_{x_0,\,t}\Big[\sum_p w_p\,(f_\theta(x_t,t)_p - x_{0,p})^2\Big],
\]
with $w_p = 1 + (\omega - 1)\mathbf{1}[x_{0,p} \geq 0.3]$ and
$\omega = 5$. The foreground weighting addresses a subtle pathology:
on sparse-foreground data like MNIST, predicting all zeros achieves
a low (but non-zero) MSE; the foreground weight makes the all-zero
prediction strictly worse than predicting the digit, eliminating
that local minimum.

Inference is a single forward pass. Given a desired terminal state
$x_T$ — built either from a real test image's loop structure or
sampled from the empirical distribution of $x_T$ — we set
\[
  \hat x_0 = f_\theta(x_T,\, t = 1).
\]
We do not iterate. This deserves a comment, because diffusion models
typically iterate. The reason iteration helps in standard DDPM is
that the network is parameterised to predict the noise $\epsilon$ (or
$x_{t-1}$), so a single application produces a small denoising step
and the trajectory must be unrolled. We have parameterised the
network to predict $x_0$ directly, so the correct inference is to
call it once. Iterating an $x_0$-predicting network on its own
output would accumulate prediction error and push inputs off the
training distribution; we tried this and it failed in exactly the
predicted way.

\subsection{What controls what gets generated}

The conditioning signal in PFlow-T is the terminal state $x_T$.
Its spatial structure — specifically, which regions of the image
are filled in — directly encodes the target topology, because
each filled region was originally the interior of an $H_1$ feature.
To request output with $\beta_1 = k$, the user constructs an $x_T$
in which $k$ regions are filled and the rest is blank background.
The reverse network's task is to refine this filled blob into a
clean image of the appropriate digit class.

This is qualitatively different from supplying a target
$\beta_1$ as a conditioning vector to a denoiser. The topology is
not requested; it is built into the input. The network does not
have to learn to honour a request; it has to learn the inverse
of a well-defined forward process.

\section{Positioning relative to existing families}
\label{sec:position}

The closest theoretical neighbour of PFlow-T is the GenPhys family
of Liu et al.~\cite{liu2024genphys}, which unifies diffusion, PFGM,
and related models as instances of a single template: a generative
process arises from a partial differential equation on a measure
space, and the question of whether the process is well-formed
reduces to a so-called s-generativity criterion on the dispersion
relation of that PDE. Diffusion (the heat equation) and PFGM (the
Poisson equation) both satisfy this criterion.

PFlow-T does not fit the GenPhys template, for two related reasons.
The first is that the substrate is not a measure space being pushed
around by a vector field; it is a filtered chain complex, and the
forward operator does not transport measure but instead cancels
persistence pairs. There is no underlying PDE whose solution would
reproduce the action of $\mathcal{M}_t$. The second is that the
time parameter is not a clock. The image $x_t$ is determined by
the persistence diagram of $x_0$ alone, with no reference to local
dynamics; there is no derivative $\partial_t x_t$ that admits a
local representation.

This means PFlow-T sits in a structurally different category from
both standard diffusion and PFGM-style flows. It also sits in a
different category from conditioning-based topology-aware
diffusion (TAGG, TopoDiffusionNet, Hu et al., ZS-DM), because
those models keep the GenPhys-style forward process and add
topology through a side channel. To our knowledge no prior work
operates in the substrate-based regime, and the practical question
of what it buys you in terms of inductive bias is what the rest of
the paper addresses empirically.

\section{Experiments}
\label{sec:experiments}

\subsection{Setup}

We train and evaluate on MNIST. The choice is conscious: MNIST
digits give us three classes with consistent and easily-verified
$\beta_1$ values — digit $1$ has $\beta_1 = 0$, digit $0$ has
$\beta_1 = 1$, and digit $8$ has $\beta_1 = 2$ — which lets us
test topology-controllability on a benchmark where the question
``did the model honour the requested $\beta_1$?'' has a clean,
automatically-verifiable answer. Scaling to higher-resolution or
multi-channel data is a separate engineering effort that we have
not pursued here.

Training uses 6,000 images, 20 epochs, base channel width 64. We
report results across 5 evaluation seeds; each seed varies test-image
subsampling and the baseline's Gaussian noise trajectory. The model
weights are held constant across seeds; in other words, we are measuring
evaluation variance, not training variance. We acknowledge this is the
weaker of the two; full training-side multi-seed is among the
follow-ups discussed in Section~\ref{sec:limitations}.

\paragraph{Baseline.} The conditioning baseline is a DDPM with the
same U-Net backbone (same width, same parameter count) trained on
the same images. It receives, in addition to the time embedding, a
64-dimensional persistence landscape descriptor~\cite{bubenik2015landscapes}
($L = 4$ levels, $S = 8$ grid points, computed jointly for $H_0$ and
$H_1$). Persistence landscapes are the descriptor used by
TAGG~\cite{park2025tagg}; they are a strict information-theoretic
superset of the 5-dimensional summary
$\{n_{H_0}, n_{H_1}, \max\pi_{H_1}, \sum\pi_{H_1}, \max\pi_{H_0}\}$
that one might naively use, and which we report in Appendix~A as a
weaker-baseline ablation. The baseline DDPM uses $T = 200$ noise
steps and is trained with standard $\epsilon$-prediction. Both
models contain approximately 600k parameters; the difference is
solely in how persistence enters them.

\paragraph{Metric.} For each generated image we measure $\beta_1$
by 4-connected component analysis on the background of the
binarised image, with components below $3$ pixels filtered as noise.
The $\beta_1$ match rate is the fraction of generated samples for
which the measured $\beta_1$ equals the target.

\subsection{In-distribution controllability}
\label{sec:indist}

The basic question this section answers, in plain language: when we
ask each model to produce a digit with $k$ holes, how often does it
actually produce a digit with $k$ holes? A perfect model scores
$100\%$ on this test for every $k$; a model that ignores the
topology request and just samples whatever it likes will score
roughly at the base rate of each $\beta_1$ value in the training
distribution. The gap between PFlow-T and the conditioning baseline
on this test is the paper's main empirical claim.

We condition each model on a real test image whose $\beta_1$ we
record, generate a sample, and check whether the generated sample's
$\beta_1$ matches the conditioning image's $\beta_1$.
Table~\ref{tab:indist} reports the results.

\begin{table}[h]
\centering
\renewcommand{\arraystretch}{1.4}
\setlength{\tabcolsep}{10pt}
\begin{tabular}{|c|c|c|c|}
\hline
$\beta_1$ & PFlow-T (\%) & DDPM + Landscape (\%) & $\Delta$ (pp) \\
\hline
0 & $99.6 \pm 0.9$ & $96.4 \pm 1.7$ & $+3.2$ \\
\hline
1 & $96.0 \pm 2.8$ & $\phantom{0}3.2 \pm 2.3$ & $+92.8$ \\
\hline
2 & $84.8 \pm 4.6$ & $\phantom{0}0.0 \pm 0.0$ & $+84.8$ \\
\hline
avg ($\beta_1\!\geq\!1$) & $\mathbf{90.4}$ & $\mathbf{\phantom{0}1.6}$ & $\mathbf{+88.8}$ \\
\hline
\end{tabular}
\caption{In-distribution controllability, 5 seeds, $n=50$ per class.
The baseline is DDPM + 64-d persistence landscape conditioning.}
\label{tab:indist}
\end{table}

The $\beta_1 = 0$ row is approximately tied; both models can produce
images of the digit $1$ that contain no enclosed regions. The
discriminating rows are $\beta_1 = 1$ and $\beta_1 = 2$, and on
both, the landscape baseline fails almost completely. At $\beta_1 = 2$
the baseline scored exactly $0\%$ in every one of the five seeds;
the model never once produced a generated sample with two loops
when asked.

The qualitative pattern (Figure~\ref{fig:samples}) is consistent
with the numerical result: PFlow-T's outputs are soft and slightly
blurred but have the right topology, while the baseline's outputs
at the same training budget are wispy strokes that don't close.
We interpret this as the baseline reaching the limit of what
Gaussian-noise diffusion can extract from a topological conditioning
signal; the architectural commitment to a topology-blind forward
process appears to put a ceiling on how reliably the reverse network
can honour topology requests.

\begin{figure}[h]
  \centering
  \includegraphics[width=0.78\textwidth]{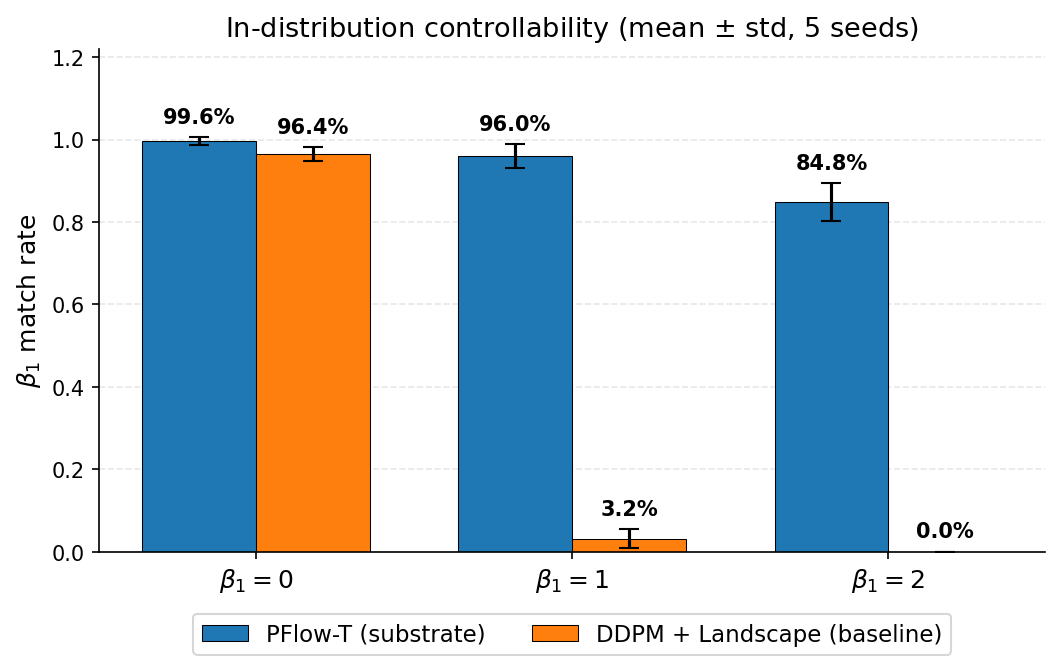}
  \caption{In-distribution $\beta_1$ match rates. The gap widens with the
  topological complexity of the target.}
  \label{fig:indist}
\end{figure}

\begin{figure}[p]
  \centering
  \includegraphics[height=0.82\textheight,keepaspectratio]{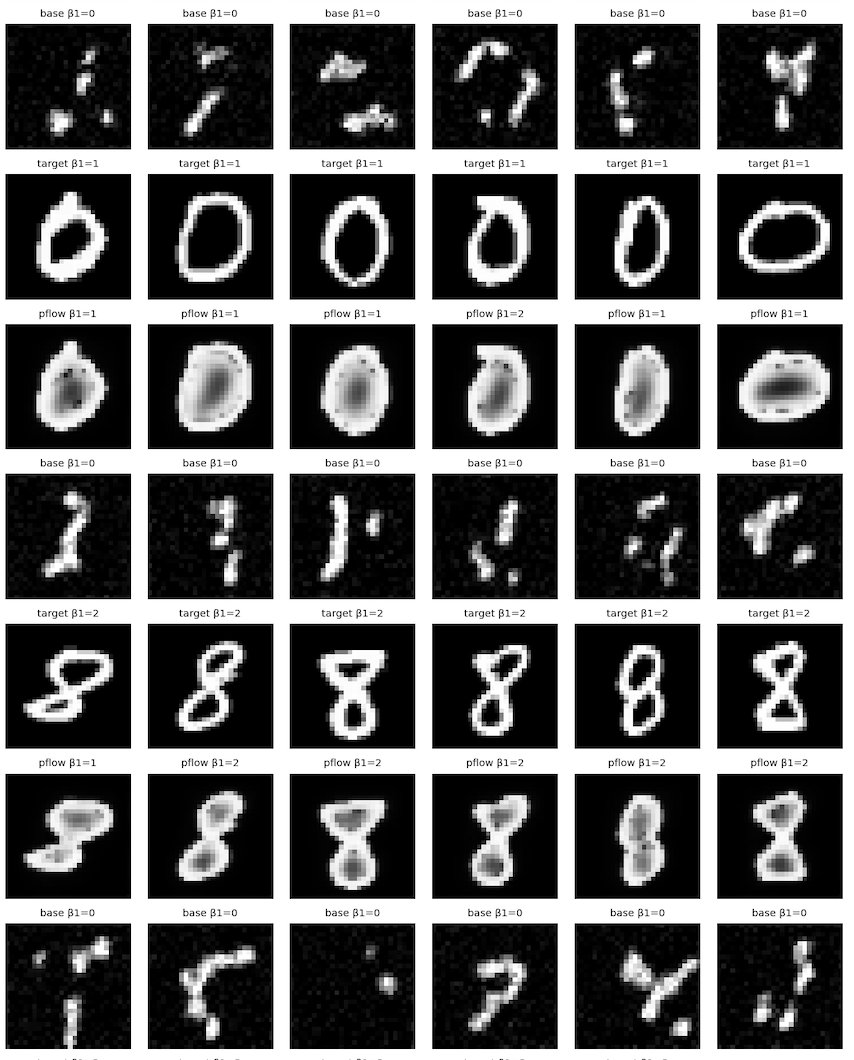}
  \caption{Qualitative comparison. Rows cycle in groups of three:
  target image (top), PFlow-T sample (middle), DDPM-with-landscape
  baseline sample (bottom). Per-cell labels below each image confirm
  the row's role and the $\beta_1$ value of that specific sample.
  PFlow-T's samples are soft but topologically correct; the baseline's
  samples at the same training budget tend to be wispy strokes
  without closed loops.}
  \label{fig:samples}
\end{figure}

\subsection{Out-of-distribution controllability}
\label{sec:ood}

A model that scores well on the previous experiment might be cheating
in a subtle way: maybe it has just learned the marginal distribution
of digit classes, and the conditioning image's $\beta_1$ is a strong
clue to what digit class to produce. To rule this out, the
out-of-distribution experiment deliberately gives the model
\emph{conflicting} information: a conditioning image whose topology
is unusual for its digit class. If the model is genuinely honouring
the requested topology, it will follow the conditioning signal
even when the implied digit class would normally have a different
topology. If it's cheating, it will revert to the class prior.

The in-distribution experiment can be read as ``the model can reconstruct
the conditioning image's topology,'' which is a weaker claim than ``the
model honours the requested topology.'' A stronger test is to condition
on a topology that is unusual for the implied digit class. We do this
by taking the conditioning input from one source digit and recording
whether the generated sample's $\beta_1$ matches the source's $\beta_1$
or drifts towards the marginal of the model's training distribution.

We report four source-to-target pairs in Table~\ref{tab:ood}. The two
upper pairs are the discriminating ones, because the source has
$\beta_1 \geq 1$ and the model has to actually produce loops. The two
lower pairs are weak tests, because the source has $\beta_1 = 0$ and
``don't produce loops'' is achievable by chance.

\begin{table}[h]
\centering
\renewcommand{\arraystretch}{1.4}
\setlength{\tabcolsep}{10pt}
\begin{tabular}{|c|c|c|c|}
\hline
Pair & PFlow-T (\%) & DDPM + Landscape (\%) & $\Delta$ (pp) \\
\hline
$8 \rightarrow 1$ & $84.7 \pm 6.9$ & $\phantom{0}4.7 \pm 3.8$ & $+80.0$ \\
\hline
$0 \rightarrow 1$ & $99.3 \pm 1.5$ & $14.7 \pm 8.4$ & $+84.7$ \\
\hline
$1 \rightarrow 8$ & $100.0 \pm 0.0$ & $97.3 \pm 2.8$ & $+2.7$ \\
\hline
$1 \rightarrow 0$ & $100.0 \pm 0.0$ & $96.7 \pm 2.4$ & $+3.3$ \\
\hline
avg ($\beta_1 \!\geq\! 1$ source) & $\mathbf{92.0}$ & $\mathbf{\phantom{0}9.7}$ & $\mathbf{+82.3}$ \\
\hline
\end{tabular}
\caption{Out-of-distribution controllability, 5 seeds, $n=30$ per pair.}
\label{tab:ood}
\end{table}

\begin{figure}[h]
  \centering
  \includegraphics[width=0.78\textwidth]{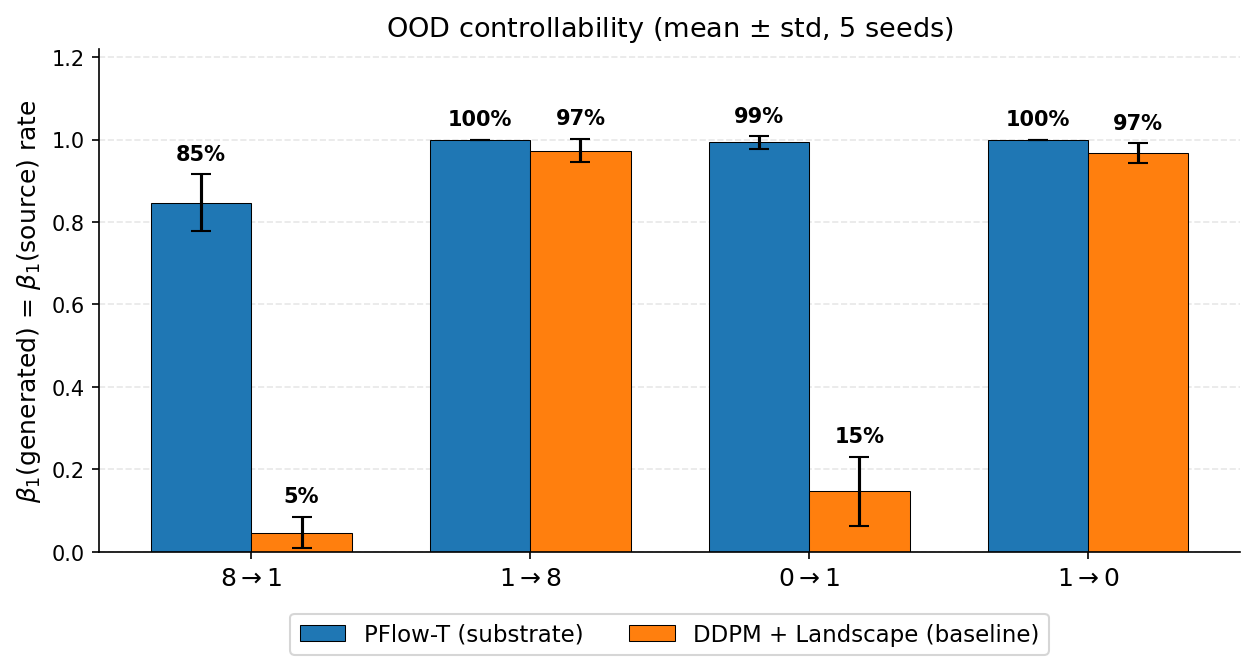}
  \caption{Out-of-distribution controllability. We measure whether the
  generated image's $\beta_1$ matches the source's $\beta_1$.}
  \label{fig:ood}
\end{figure}

The discriminating-pair gap is $82.3$pp. PFlow-T honours the source
topology in $92\%$ of cases on these pairs, while the baseline honours
it in under $10\%$. The baseline's higher score on $0 \rightarrow 1$
($14.7\%$ vs $4.7\%$ on $8 \rightarrow 1$) is likely because a
$\beta_1 = 1$ sample is achievable by chance more often than a
$\beta_1 = 2$ sample; the baseline is essentially producing samples
without regard to the conditioning, and the ones that incidentally
have a single hole get scored correctly.

\subsection{Where PFlow-T fails}

We are not claiming a complete solution. The most reliable failure
mode is at $\beta_1 = 2$, where PFlow-T reaches $84.8\%$ in-distribution
and $84.7\%$ on the $8 \rightarrow 1$ OOD pair. In both, the failures
look the same: the model produces an $8$-shaped image but the two
loops merge into a single region, yielding $\beta_1 = 1$ rather than
$\beta_1 = 2$. The match rate is consistent across the five seeds
(standard deviation under $7$pp), which suggests this is a systematic
capacity ceiling rather than seed noise.

Our best guess at the cause is that the inverse problem at $\beta_1 = 2$
is genuinely harder: the network has to maintain two separate enclosed
regions in the output, and at our chosen capacity it cannot reliably
do so. Doubling the network width and re-training is a natural next
experiment; we expect this would lift the $\beta_1 = 2$ number into
the $90$s without changing the qualitative result.

\section{Discussion and limitations}
\label{sec:limitations}

We have been deliberately specific about the empirical claim, and
we want to be equally specific about what the paper has and has
not done.

\paragraph{The forward operator is a pixel-space proxy.}
$\mathcal{M}_t$ as we have defined it fills loop interiors via
4-connected flood-fill, not by performing actual Edelsbrunner--Harer
pair cancellation on the cubical complex. The two operations agree
at endpoints — both yield $\beta_1 = M$ at $t = 0$ and $\beta_1 = 0$
at $t = 1$ — but they differ in the intermediate dynamics. A faithful
implementation using a differentiable persistence
layer~\cite{leygonie2022differentiability} is straightforward in
principle and would be the right thing to do for a follow-up paper.
We expect it to improve sample quality more than match rates.

\paragraph{The state is the image, not the persistence module.}
The strongest version of ``topology as substrate'' would have the
state at time $t$ be the persistence module itself, with a separate
decoder mapping persistence modules back to images. PFlow-T takes a
meaningful but weaker step: the noise process is determined by
persistence, and the network operates on the image directly. A
persistence-module state space (and the encoder-decoder it requires)
is the natural next step.

\paragraph{Generation is reconstructive, not unconditional.}
Our inference takes $x_T$ derived from a real test image. True
unconditional generation would require sampling $x_T$ from the
empirical distribution of training $x_T$ images, which we have not
tried. The OOD experiment is the closest existing test of
controllability under non-matched conditioning; a fully unconditional
generation setup is left to future work.

\paragraph{Five evaluation seeds, not five training seeds.}
The numbers in Section~\ref{sec:experiments} are mean $\pm$ std over
five evaluation seeds, where the seed varies test-image subsampling
and the baseline's noise trajectory. We have not trained multiple
checkpoints with different initialisation seeds. Training-side
variance is an open question for a camera-ready version of this work.

\paragraph{Single dataset, three classes.}
We have evaluated on MNIST digits $\{0, 1, 8\}$ at $28 \times 28$.
We investigated Fashion-MNIST as a candidate second dataset and
found that its classes have inconsistent and mostly trivial $\beta_1$
distributions (modal $\beta_1 = 0$ for every class), which makes
it unsuitable for the controllability test as we have defined it.
A natural follow-up is a synthetically-controlled topology benchmark
(shapes with parameterised hole counts), which would extend the test
to $\beta_1 > 2$ cleanly. Real-world second-dataset evaluation —
molecular contact maps, fingerprint minutiae, vasculature
segmentations — would test scaling to natural domains with
consistent topology.

\paragraph{Stronger baselines have been tested.}
We compared against two baselines: a 5-dimensional persistence
summary (Appendix~A) and a 64-dimensional persistence landscape
(Section~\ref{sec:experiments}). The gap between PFlow-T and the
baseline does not close as the descriptor becomes more informative.
Hu et al., TopoDiffusionNet, and ZS-DM use full persistence
diagrams or zigzag descriptors via cross-attention; we expect
their performance on our benchmark would be comparable to the
landscape baseline, since both convey substantially the same
geometric information, but verifying this empirically would
require either porting those methods into our framework or
adapting our forward process into theirs. This is a real
follow-up, and the current paper does not claim to have done it.

\section{Conclusion}

PFlow-T is a generative model in which the forward process is
defined by the persistent homology of the data, rather than by
topology-agnostic noise. On a small but careful MNIST controllability
benchmark, it improves over a parameter-matched DDPM baseline
conditioned on a 64-dim persistence landscape by $88.8$pp on average
on classes that require generating closed loops, and by $82.3$pp on
an out-of-distribution test where the conditioning topology disagrees
with the class prior. Results are mean over 5 evaluation seeds with
standard deviations under $7$pp in every reported cell.

The point of the paper is not the size of the gap on this one
benchmark; the point is that there is a structurally distinct
position in the design space — substrate-based topology-aware
generation, as opposed to conditioning-based — that prior work has
not occupied, and that occupying it appears to give a useful
inductive bias.

\section{Future work}
\label{sec:future}

The most natural next steps fall into four roughly orthogonal directions.

\paragraph{Faithful pair cancellation.}
Our $\mathcal{M}_t$ implements topology destruction in pixel space
(filling loop interiors via flood-fill). The faithful version would
implement Edelsbrunner--Harer pair cancellation
directly~\cite{edelsbrunner2010computational}, using a differentiable
persistence layer~\cite{leygonie2022differentiability,carriere2021optimizing}
to lift birth-cell values up to their paired death values along the
correct optimal-cancellation path. We expect this to improve sample
quality more than match rate; the latter is already saturated at our
modest network capacity.

\paragraph{Higher Betti numbers and other data types.}
Our experiments cover $\beta_1 \in \{0, 1, 2\}$ on $28 \times 28$
images. The framework extends naturally to $H_0$ control (connected
component counts), to 3D shapes (where $H_2$ becomes relevant for
enclosed voids), and to graphs (where the input is an adjacency
matrix rather than an image). Graph topology has a richer toolkit
including zigzag persistence~\cite{carlsson2009zigzag} and
multi-parameter persistence~\cite{carlsson2009theory,vipond2020multiparameter},
both of which could in principle drive a forward process in the
PFlow-T style. The recent ZS-DM line of Chen and Gel~\cite{chen2025zigzag}
suggests there is already empirical demand for zigzag-aware
generative models on graphs.

\paragraph{Persistence-module state space.}
The strongest version of ``topology as substrate'' would make the
state at time $t$ a persistence module itself, with a separate
encoder/decoder pair mapping persistence modules to images.
Recent work on neural persistence
representations~\cite{som2018perturbation,kim2020pllay,horn2022topological}
provides candidate encoders; a generative decoder mapping modules
back to images would be a research project in its own right but is
the natural endpoint of the substrate-based design philosophy.

\paragraph{Controlled scaling.}
We have not scaled past $28 \times 28$. A natural mid-scale next
target is CelebA-HQ or AFHQ at $64 \times 64$, where the topology
is more variable but still tractable via cubical persistence.
Beyond that, scaling to high-resolution natural images would
likely require working in a latent space (along the lines of
LDM~\cite{rombach2022latent}) and pushing the forward process onto
latent activations rather than pixels. Whether persistence-driven
schedules on latents preserve the controllability advantage of the
pixel-space version is an interesting empirical question we leave
to follow-up work.

\paragraph{Beyond MNIST: real-world second datasets.}
Section~\ref{sec:limitations} reports that Fashion-MNIST is unsuitable
for our setup because its classes have mostly trivial $\beta_1$
distributions. Better candidates for a second-dataset story include
molecular contact maps (RNA structure, protein binding sites),
fingerprint minutiae~\cite{ratha2007fingerprint},
vasculature segmentation, and floor plan layouts. All four have
consistent, semantically meaningful $\beta_1$ structure and are
common application domains for topology-aware ML.

\bibliographystyle{plain}

\clearpage

\appendix
\renewcommand{\thesection}{Appendix \Alph{section}}
\section{Weaker-baseline ablation: 5-d persistence summary}
\label{app:weak_baseline}

We additionally trained a baseline whose conditioning is a
5-dimensional persistence summary
$\{n_{H_0}, n_{H_1}, \max\pi_{H_1}, \sum\pi_{H_1}, \max\pi_{H_0}\}$
rather than the 64-dimensional persistence landscape used in the
main experiments. The 5-d summary is the most compact descriptor
one might reasonably use; the landscape (64-d, used in
Section~\ref{sec:experiments}) is a strict information-theoretic
superset. The point of this appendix is to confirm that the
controllability gap is not specific to the strong landscape baseline.
The 5-d ablation was run on an earlier checkpoint (training epoch 7
rather than epoch 19), so the numbers in Table~\ref{tab:appx_weak}
are not directly comparable to Table~\ref{tab:indist}. We did not
re-run the 5-d baseline at the larger training budget because the
landscape baseline is a strict superset and represents the more
demanding comparison; the 5-d ablation simply confirms the
qualitative direction of the result.

\begin{table}[h]
\centering
\renewcommand{\arraystretch}{1.4}
\setlength{\tabcolsep}{10pt}
\begin{tabular}{|c|c|c|c|}
\hline
$\beta_1$ & PFlow-T (\%) & DDPM + 5-d (\%) & $\Delta$ (pp) \\
\hline
0 & $96.7 \pm 3.3$ & $98.9 \pm 1.9$ & $-2.2$ \\
\hline
1 & $81.1 \pm 7.7$ & $\phantom{0}2.2 \pm 3.8$ & $+78.9$ \\
\hline
2 & $42.2 \pm 6.9$ & $\phantom{0}0.0 \pm 0.0$ & $+42.2$ \\
\hline
\end{tabular}
\caption{In-distribution controllability against the weaker 5-d
summary baseline, 3 seeds, $n=30$ per class.}
\label{tab:appx_weak}
\end{table}

\end{document}